\documentclass[a4paper,11pt]{article}

\usepackage{fullpage}
\usepackage{authblk}

\title{Markov Decision Processes with Time-Varying Geometric Discounting}

\author[1]{Jiarui Gan}
\author[2]{Annika Hennes}
\author[3]{Rupak Majumdar}
\author[3]{\\Debmalya Mandal}
\author[3]{Goran Radanovic}

\affil[1]{University of Oxford}
\affil[2]{Heinrich-Heine-University D\"usseldorf}
\affil[3]{Max Planck Institute for Software Systems}

\date{}

\usepackage{natbib}  
\setcitestyle{authoryear}
\setcitestyle{brackets=round}
\setcitestyle{aysep={}}

\usepackage{subcaption}
\usepackage{amsmath}
\usepackage{amsfonts}
\usepackage{mathtools}
\usepackage{amsthm}
\usepackage{thmtools, thm-restate}

\declaretheorem[numberwithin=section]{theorem}
\declaretheorem[sibling=theorem]{lemma, definition, observation, proposition}

\DeclarePairedDelimiter\abs{\lvert}{\rvert}%
\DeclarePairedDelimiter\norm{\lVert}{\rVert}%

\usepackage{cleveref}

\usepackage{adjustbox}
\usepackage[ruled, vlined]{algorithm2e}

\SetCommentSty{mycommfont}
\SetAlFnt{\small}
\usepackage{etoolbox}
\makeatletter
\patchcmd{\@algocf@start}
{-1.5em}
{0pt}
{}{}
\makeatother

\usepackage{tikz}
\usetikzlibrary{shapes, shapes.misc, automata, positioning, arrows, arrows.meta}
\tikzset{
	state/.style={
		circle,
		draw=black,
		fill=gray!30,
		text opacity=1,
		inner sep=0pt,
		minimum size=23pt,
		font=\small},
	->, 
	>=Stealth, 
	node distance=1.8cm, 
}

\newcommand{\calM}{\mathcal{M}}
\newcommand{\calG}{\mathcal{G}}
\DeclareMathOperator{\argmax}{arg\,max}

\newcommand{\E}{\mathbb{E}}
\newcommand{\poly}{poly}

\newcommand{\valit}{\textsc{ValIt}}
\newcommand{\sstart}{s_{\text{start}}}

\newcommand{\expct}[1]{\mathbb{E}\bigl[#1\bigr]}
\newcommand{\abso}[1]{\left| #1 \right|}

\newcommand{\vargamma}{g}
\newcommand{\ppi}{{\boldsymbol\pi}}
\newcommand{\tgamma}{\tilde{\gamma}}

\begin{document}
	
	\maketitle
	
	\begin{abstract}
		Canonical models of Markov decision processes (MDPs) usually consider geometric discounting based on a constant discount factor. While this standard modeling approach has led to many elegant results, some recent studies indicate the necessity of modeling time-varying discounting in certain applications. This paper studies a model of infinite-horizon MDPs with time-varying discount factors. We take a game-theoretic perspective---whereby each time step is treated as an independent decision maker with their own (fixed) discount factor---and we study the {\em subgame perfect equilibrium} (SPE) of the resulting game as well as the related algorithmic problems. We present a constructive proof of the existence of an SPE and demonstrate the EXPTIME-hardness of computing an SPE. We also turn to the approximate notion of $\epsilon$-SPE and show that an $\epsilon$-SPE exists under milder assumptions. An algorithm is presented to compute an $\epsilon$-SPE, of which an upper bound of the time complexity, as a function of the convergence property of the time-varying discount factor, is provided. 
	\end{abstract}
	
	\section{Introduction}
	
	Ever since Samuelson's foundational work introduced the discounted utility theory \citep{samuelson1937note}, discounted utility models have played a central role in sequential decision making. Building on earlier work that recognized the influence of time on individuals' valuations of goods (e.g., see \citep{rae1905sociological,jevons1879theory,von1922capital,fisher1930theory} and the discussion in \citep{loewe2006development}), Samuelson proposed a utility model in which a decision maker attempts to optimize the discounted sum of their utilities with a constant discount factor applied in every time step; this is known as geometric or exponential discounting.
	
	Geometric discounting leads to many elegant and well-known results.
	In the context of Markov decision processes (MDPs) \citep{Puterman1994}, it results in the decision maker's preferences over the policies being invariant over time. Moreover, it is key to the existence and polynomial-time computability of an optimal policy. These results have contributed greatly to the popularity and wide applicability of the MDPs.
	Nevertheless, in many applications, in particular those pertaining to human decision making under uncertainty, time-varying discount factors are essential for capturing long-run utilities. 
	For example, it is shown in laboratory settings that human decision makers often exhibit time-inconsistent behavior: people prefer \$50 in three years plus three weeks to \$20 in three years, yet prefer \$20 now over \$50 in three weeks \citep{green1994temporal}.
	Such behaviors are better explained through time-varying discount factors.
	Unfortunately, many of the aforementioned results break with time-varying discounting.  
	As \citet{strotz1955myopia} showed, geometric discounting with a constant discounting factor is the only discount function that satisfies dynamic- or time-consistency. 
	
	In this paper, we study a model of infinite-horizon MDPs with time-varying geometric discounting. 
	Our model seizes the idea of geometric discounting, but generalizes the discount factor to a function of time. 
	In each time step, the function produces a discount factor, and the agent's incentive is defined by the geometrically discounted sum of its future rewards with respect to this discount factor.
	Hence, the agent aims at optimizing a different objective in each time step. 
	This changing incentive gives rise to a game-theoretic approach, proposed and studied in a series of works in the literature \citep{strotz1955myopia,pollak1968consistent,peleg1973existence,lattimore2014general,jaskiewicz2021markov,lesmana2021subgame}.
	Via this approach, the behavior of the sole agent in the process is interpreted as playing against its future selves in a sequential game. Analyzing the {\em subgame perfect equilibrium} (SPE) is therefore a naturally associated task, which we aim to address in this paper.
	
	\begin{figure*}[t]
		\centering
		\begin{subfigure}[t]{0.32\textwidth} 
			\centering
			\includegraphics[width=\textwidth]{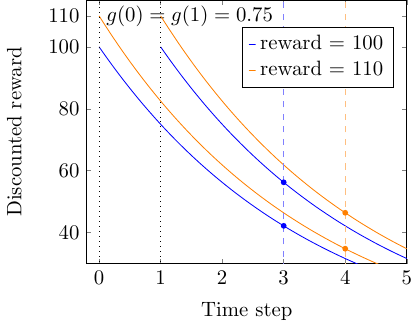}
			\caption{Standard geometric discounting.} 
			\label{fig:example-decision-problem:standard-discounting}
		\end{subfigure} 
		\hfill
		\begin{subfigure}[t]{0.32\textwidth}
			\centering
			\includegraphics[width=\textwidth]{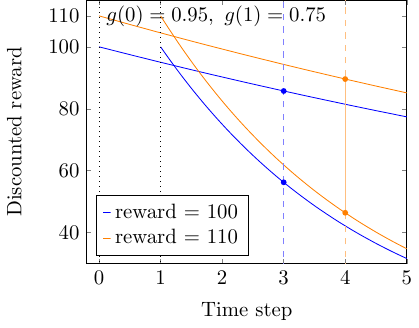}
			\caption{Time-varying discounting.} 
			\label{fig:example-decision-problem:time-varying-discounting}
		\end{subfigure} 
		\hfill
		\begin{subfigure}[t]{0.32\textwidth}
			\centering
			\raisebox{11mm}{
				\includegraphics[width=0.95\textwidth]{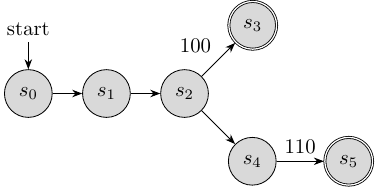}
			}
			\caption{MDP (rewards are $0$ if unspecified)}
			\label{fig:example-decision-problem:mdp}
		\end{subfigure}
		\caption{Time-consistent vs. time-inconsistent discounting. 
			The blue (resp., orange) curves in (a) and (b) show how the agent values getting a reward of $100$ (resp., $110$) in future time steps. We examine the player's preferences at the beginning ($s_0$) and one time step later ($s_1$). 
			The colored dots correspond to values of the two possible outcomes of the MDP in (c).
		}
		\label{fig:example-decision-problem}
	\end{figure*}
	
	More concretely, \Cref{fig:example-decision-problem} presents an example that compares the behavior of this model to the standard model of geometric discounting, illustrating how time-varying geometric discounting can lead to time-inconsistency.
	In the example, the agent
	has to decide between getting a reward of $100$ at time step 3 (option A) and getting a slightly increased reward of $110$ one time step later (option B). 
	This decision problem is captured by the MDP in \Cref{fig:example-decision-problem:mdp}.
	Under the standard geometric discounting, the preference order of these two outcomes does not change over time: as illustrated in \Cref{fig:example-decision-problem:standard-discounting}, an agent that discounts its future with a constant factor $0.75$ would always prefer option A, no matter at time step 0 or 1.
	This is not anymore the case with time-varying discounting. 
	An agent who applies a discount factor of $0.95$ is more farsighted and prefers the higher but delayed reward (i.e. option B). But if the agent becomes more myopic one time step later on by applying a reduced discount factor of $0.75$, the preference order would change, and the agent would no longer want to stick to its initial plan. This situation is illustrated in \Cref{fig:example-decision-problem:time-varying-discounting}.
	
	Time-inconsistent behavior may result from different forms of discounting.
	It may arise from a consistent way of planning but an inconsistent treatment of future time steps. 
	In hyperbolic discounting, for example, the agent assigns a fixed sequence of varying discount factors to future time steps relative to the current step.
	In other words, the agent plans with consistency (using the same sequence over time) but treats the future inconsistently (discounting different time steps differently). 
	In contrast, in our model, the agent discounts the future using a constant factor, but this factor might change over time. Human beings, for example, may start with very low discount factors when they are young, but increasingly think more about the future as they grow older.
	Conceivably, a young person, probably through observing this in elder people, knows that their way of discounting will change when they go into middle age. But still, they cannot do anything about their urge for immediate rewards.
	Likewise, the agent in our model knows how its rate changes in the future and tries to find a compromise between all the different preferences. This motivates the use of the SPE as our solution concept.

	\subsection{Contributions}
	
	Besides introducing a model of MDP with time-varying discounting, we make the following technical contributions. 
	\begin{itemize}
		\item We present a constructive proof for the existence of an SPE.
		Our proof differs from another non-constructive approach in the literature \citep{lattimore2014general} which uses the compactness of the underlying space to argue about convergence. 
		
		\item
		From our constructive proof, an algorithm for computing an SPE can be readily extracted. Meanwhile, we demonstrate that the problem of computing an SPE is EXPTIME-hard even in restricted settings. 
		
		\item 
		In order to circumvent some of the assumptions needed to construct an exact SPE, we turn to the relaxed notion of $\epsilon$-SPE. We show that an $\epsilon$-SPE exists under strictly milder assumptions and present an algorithm to compute an $\epsilon$-SPE.
		Using a continuity argument of the value functions, we also derive an upper bound on the time complexity of the algorithm, as a function of the convergence property of the time-varying discount factors.
	\end{itemize}

	\subsection{Related Work}
	
	A large body of experimental evidence suggests that human behavior is not characterized by geometric discounting with a constant discount factor. 
	Empirical findings that do not support the hypothesis that discounting is consistent over time have been reported \citep{thaler1981some,benzion1989discount,redelmeier1993time,green1994temporal,kirby1995preference,millar1984self}, implying dynamic inconsistency of human preferences.
	Prior work has also proposed and studied different forms of discounting, such as hyperbolic~\citep{herrnstein1961relative,ainslie1992picoeconomics,loewenstein1992anomalies} or quasi-hyperbolic discounting~\citep{phelps1968second,laibson1997golden}, which are considered to be more aligned with human behavior~\citep{ainslie1975specious,green1994discounting,kirby1997bidding}.
	Interpretations of discounting functions as uncertainty over hazard rates were also proposed \citep{sozou1998hyperbolic,dasgupta2005uncertainty}. 
	
	Focusing on sequential decision making under uncertainty, this paper closely relates to the line of work that studies non-geometric discount factors and dynamic inconsistency in Markov decision processes and stochastic games~\citep{shapley1953stochastic,alj1983dynamic,Puterman1994,nowak2010noncooperative,jaskiewicz2021markov,lesmana2021subgame}. 
	Some recent works studied dynamic inconsistency using a game-theoretic framework akin to the ones by~\citet{strotz1955myopia,pollak1968consistent,peleg1973existence}, where their focus is on the existence of an equilibrium in randomized stationary Markov perfect strategies \citep{jaskiewicz2021markov}, or an SPE in a finite horizon setting \citep{lesmana2021subgame}.
	Arguably, the closest work to ours is the work of \citet{lattimore2014general}, which considers ``age-dependent'' (time-varying) geometric discounting functions. The characterization results therein prove the existence of an SPE in the resulting game. Our results are complementary: we provide a constructive proof of the existence result and additionally study the computational complexity of the problem setting.

	MDP-based settings similar to ours have also been considered in reinforcement learning \citep{sutton1995td,sutton2011horde,white2017unifying,pitis2019rethinking,fedus2019hyperbolic}.
	Increasing the efficacy of learning by using multiple discount factors has been explored previously \citep{burda2018exploration,romoff2019separating}. It is also worth mentioning settings that use a weighted reward criterion (e.g., \citet{filar2012competitive}), where the objective can be expressed as the weighted sum of two value functions with different discount factors. 
	
	\section{The Model}
	
	We consider an infinite horizon MDP $\calM = (S, A, R, P, \sstart, \gamma)$, where
	$S$ is a finite state space of the environment, with $\abs{S} = n$, and  
	$A = \bigcup_{s\in S} A_s$ is a union of finite action spaces, with each $A_s$ being the set of actions available in state $s$.
	Moreover, $R \colon S \times A \to \mathbb{R}$ is a reward function, such that when action $a$ is taken in state $s$, a reward $R(s,a)$ will be generated, and the state of the environment transitions according to the transition function $P\colon S \times A \to S $, with probability $P(s,a,s')$ to another state $s' \in S$.
	Finally, $\sstart \in S$ is a starting state, and $\gamma \in [0,1)$ is a discount factor that is applied for defining the {\em cumulative reward} of a policy $\pi$, i.e., the discounted sum of rewards obtained over an infinite horizon:
	\begin{equation}
		\label{eq:cul-reward}
		\E\left[\sum_{t=0}^{\infty} \gamma^{t}\cdot R(s_{t}, a_{t}) \biggm\vert s_0 = \sstart, \pi \right]\thinspace,
	\end{equation} 
	where the expectation is over the trajectory $(s_t, a_t)_{t=0}^\infty$ generated by starting from $\sstart$ and following $\pi$ subsequently.
	
	Two types of policies will be of interest in this paper: {\em static policies} and {\em dynamic policies}.
	A static policy $\pi: S \to A$ assigns a (deterministic) action to each state,
	so that using policy $\pi$, an agent performs action $\pi(s)$ whenever the environment is in state $s$, irrespective of the time step.
	In contrast, a dynamic policy is time-dependent and is defined as a sequence $\ppi = (\pi_t)_{t=0}^\infty$ of static policies. At each time step $t$, the static policy $\pi_t$ is employed to determine the action to take. Hence, dynamic policies are a generalization of static policies.\footnote{More generally, a static policy can also choose randomized actions, i.e., $\pi: S \to \Delta(A)$. Nevertheless, it is without loss of generality to consider only deterministic policies with respect to all the results in our paper. Hence, unless otherwise clarified, all static policies considered are deterministic ones, whereas we do allow the agent to use randomized static policies.}
	
	One may also consider policies that depend on the history (i.e., the trajectory of states and actions $s_0,a_0,s_1,a_1,\ldots,s_{t-1},a_{t-1}$ generated so far), but as it shall be clear this is unnecessary for the problems studied in this paper: the underlying process is Markovian and the agent always observes the state of the environment.
	Moreover, it is well-known that when the discount factor $\gamma$ is a constant, to maximize the cumulative reward defined in \eqref{eq:cul-reward} it suffices to consider static policies. Though this does not hold when the optimization horizon is finite \citep{shirmohammadi2019complexity} or, as we will study in this paper, when $\gamma$ varies with time.

	\subsection{Constant Discount Factor and Optimality}
	
	When a constant discount factor is applied, the optimality of a static policy with respect to \eqref{eq:cul-reward} can be characterized using the value function $V_{\gamma}^{\pi}$ defined as
	\begin{align}
		\label{eq:v-func-const-gamma}
		V_{\gamma}^{\pi}(s) 
		\coloneqq Q_{\gamma}^{\pi}(s, \pi(s))
	\end{align}
	for all $s \in S$,
	where
	\begin{align}
		\label{eq:Q-func-const-gamma}
		Q_{\gamma}^{\pi}(s, a)
		\coloneqq&\ R(s,a) + \gamma \cdot \mathbb{E}_{s' \sim P(s, a, \cdot)} V_\gamma^\pi(s') \nonumber \\
		=&\ R(s,a)  + \gamma\cdot\sum_{s'\in S}P(s,a, s')\cdot V_{\gamma}^{\pi}(s') 
	\end{align}
	is called the Q-function.
	The value of $V_{\gamma}^{\pi}(s)$ corresponds to the expected sum of rewards when starting in state $s$ and following policy $\pi$.
	A static policy $\pi^*$ is optimal if for every state $s$ and every action $a\in A$, it holds that 
	\begin{equation} 
		\label{eq:V-pi-gamma}
		V_{\gamma}^{\pi^*}(s) = \max_{a \in A_s}\ Q_\gamma^{\pi^*}(s, a)\thinspace.
	\end{equation}
	We denote by $\Pi$ the set of all static policies, and by $\Pi_{\gamma}^*$ the set of all optimal policies with respect to a constant $\gamma$.
	It is well known that $\Pi_{\gamma}^* \neq \emptyset$ for all $\gamma \in [0,1)$, and one can compute a policy in $\Pi_{\gamma}^*$ in polynomial time~\citep{rincon2003existence}.
	
	It will also be useful to introduce the notion of equivalent policies.
	Two policies are deemed equivalent if their value functions are identical for all states and discount factors.
	
	\begin{definition}[Equivalent policies]
		\label{def:equi-policies}
		Two static policies $\pi_1, \pi_2 \in \Pi$ are {\em equivalent} if for all $s\in S$ and all $\gamma \in [0, 1)$ it holds that 
		$V_{\gamma}^{\pi_1}(s) = V_{\gamma}^{\pi_2}(s)$.
		We write $\pi_1 \sim \pi_2$ if $\pi_1$ and $\pi_2$ are equivalent.
	\end{definition}

	\subsection{Time-Varying Discounting---Game-Theoretic View}

	\newlength{\mywidth}
	\setlength{\mywidth}{0.3\textwidth}
	\begin{figure*}[t]
		\centering
		\begin{subfigure}[t]{\mywidth}
			\centering
			\includegraphics[width=\linewidth]{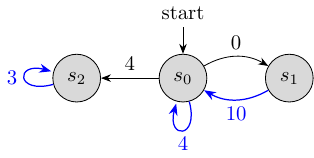}
			\caption{}\label{fig:example-gamma-1}
		\end{subfigure}
		\hfill
		\begin{subfigure}[t]{\mywidth}
			\centering
			\includegraphics[width=\linewidth]{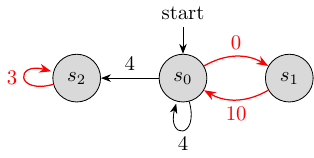}
			\caption{}
			\label{fig:example-gamma-2}
		\end{subfigure}
		\hfill
		\begin{subfigure}[t]{\mywidth}
			\centering
			\includegraphics[width=\linewidth]{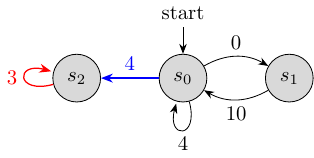}
			\caption{}
			
			\label{fig:example-spe}
		\end{subfigure}
		\caption{Optimal static policies for geometric discounting with different values of $\gamma$ and an SPE. (a) and (b) illustrate the optimal policies for $\gamma_1=0.1$ ({blue}) and $\gamma_2 = 0.8$ ({red}), respectively. (c) illustrates an SPE under time-varying discounting, with $\vargamma(0) = \gamma_1 $ and $\vargamma(i) = \gamma_2$ for all $i\geq 1$. The action chosen by player 0 is depicted in {blue}, the actions chosen by all subsequent players are in {red}.}
		\label{fig:example}
	\end{figure*}

	We generalize the above definition to {\em MDPs with time-varying discounting} (hereafter, MDPs for simplicity) by replacing the constant factor $\gamma$ by a {\em discount function} $\vargamma:\mathbb{N} \to [0,1)$, such that $\vargamma(t)$ is the discount factor the agent applies at time step $t$. 
	We will only consider discount functions that converge to a value in $[0,1]$ when $t \to \infty$ in this paper.
	
	Time-varying discounting changes the agent's incentive over time and as a result the agent behaves as if they are different agents. 
	Hence, we apply a game-theoretic view and view the MDP as a sequential game played by countably many players. Every player is associated with a time step $t\in \mathbb{N}$ and decides on a static policy $\pi_t$ to use at that particular time step. Moreover, player $t$ represents the agent's incentive at time step $t$ and cares about the subsequent cumulative reward with respect to the (constant) discount factor $\vargamma(t)$, i.e.,
	\begin{equation}
		\label{eq:u-t}
		u_t(\ppi|s) \coloneqq \E\left[\sum_{t'=t}^{\infty} \vargamma(t)^{t' - t}\cdot R(s_{t'}, a_{t'}) \biggm\vert s_t = s, \ppi \right]
	\end{equation} 
	when the environment is in state $s$ before player $t$ is to take an action, and the other players $t+1, t+2, \dots$ subsequently act according to $\pi_{t+1}, \pi_{t+2}, \dots$ given by the dynamic policy $\ppi =(\pi_t)_{t=0}^\infty$.
	In other words, each player has the same geometric-discounting-style vision as that defined in \eqref{eq:cul-reward}. 
	The function $u_t$ can be viewed as the utility function of player $t$ conditioned on $s_t$, and $\ppi$ as the players' strategy profile.
	The discount factor stays constant for this particular player, but it might be different for different players.
	We will analyze the {\em subgame-perfect equilibrium} (SPE) of the resulting game, which is the standard solution concept for sequential games \citep{osborne2004introduction}.
	
	\begin{definition}[SPE]
		A dynamic policy $\ppi = (\pi_t)_{t=0}^\infty$ is an SPE if for all $t \in \mathbb{N}$ and $s \in S$ it holds that: $u_t(\ppi|s) \ge u_t(\ppi'|s)$
		if $\pi'_i = \pi_i$ for all $i \in \mathbb{N} \setminus \{t\}$.
	\end{definition}
	
	In other words, in an SPE, from any time step $t$ onward, the players' policies form a Nash equilibrium of the subsequent subgame, no matter what $s_t$ is.
	Note that the above definition takes the same form as a Nash equilibrium of the players' policies because every player $t$ only plays at time step $t$ throughout the game.

	We can use value functions to characterize an SPE:
	a dynamic policy $\ppi^*$ is an SPE if it holds that
	\begin{align} 
		\label{eq:SPE-v-func}
		V_{\vargamma(t), t}^{\ppi^*}(s) = \max_{a \in A_s}\ Q_{\vargamma(t),t}^{\ppi^*} (s, a)\thinspace,
	\end{align}
	for all $t =0,1,\dots$ and $s \in S$,
	where for any $\ppi$ we define
	\begin{flalign}
		\label{eq:v-func-vargamma}
		V_{\gamma, t}^{\ppi}(s) &\coloneqq Q_{\gamma, t}^{\ppi}(s, \pi_t(s)), \text{ and }\\
		\label{eq:Q-func-vargamma}
		\
		Q_{\gamma, t}^{\ppi}(s, a)
		&\coloneqq R(s,a)  + \gamma\! \sum_{s'\in S}P(s,a, s') V_{\gamma, t+1}^{\ppi}(s')\thinspace.\!\!\!
	\end{flalign}
	Namely, each player $t$ has a value function $V_{\vargamma(t),i}^\ppi$ and a Q-function $Q_{\vargamma(t),i}^\ppi$ for each time step $i$, defined with respect to their own discount factor $\vargamma(t)$.
	We can make two observations below: the first observation follows by definition (i.e., \eqref{eq:u-t}), and the second holds as the dynamic policy essentially degenerates to a static one in the stated situation.
	
	\begin{observation}
		\label{obs:ut-Vgt}
		$u_t(\ppi|s) \equiv V_{\vargamma(t), t}^{\ppi}(s)$ for all $s \in S$.
	\end{observation}
	
	\begin{observation}
		\label{obs:V-constant-tail}
		Let $\ppi = (\pi_t)_{t=0}^\infty$ be a dynamic policy and $\pi \in \Pi$ be a static policy.
		If $\pi_t = \pi$ for all $t \ge T$, then
		$V_{\gamma,t}^\ppi(s) = V_{\gamma}^{\pi}(s)$ for all $t \ge T$ and all $s \in S$.
	\end{observation}
	
	We analyze the problem of computing an SPE.
	Since a solution to this problem is a dynamic policy over an infinite horizon, it is not immediately clear whether a solution admits any concise representation. We therefore consider only the first step ($t=0$) and the following decision problem:
	{\em For a given action $a \in A$, is there an SPE $\ppi$ such that $\pi_0(s_0) = a$?} (More formally, see the definition below.)
	We refer to this problem as {\sc SPE-Start}.
	
	\begin{definition}[{\sc SPE-Start}]
		An instance of {\sc SPE-Start} is given by a tuple $(\calM, a^\dag)$, consisting of an MDP $\calM = (S,A,R,P,\sstart,\vargamma)$ (with a time-varying discount function $\vargamma$) and an action $a^\dag \in A$. It is a yes-instance if there exists an SPE $\ppi$ such that $\pi_0(\sstart) = a^\dag$; and a no-instance, otherwise.
	\end{definition}
	
	It is straightforward that when $\vargamma$ is a constant function, an SPE corresponds to an optimal policy for the MDP. Yet, it appears that {\sc SPE-Start} is computationally more demanding than computing an optimal policy in a constant-discounting MDP: as we will show in the paper, {\sc SPE-Start} is EXPTIME-hard, whereas the latter is well-known to be solvable in polynomial time.

	\subsection{An Illustrating Example}
	\label{section:an-illustrating-example}
	
	We conclude our description of the model with an illustrating example.
	Consider the MDP given in \Cref{fig:example}, and two different discount factors, $\gamma_1 = 0.1$ and $\gamma_2 = 0.8$. Let $s_0$ be the starting state.
	As we are solely considering deterministic state transitions in this example, we will identify actions starting in a certain state with the state it changes to, e.g. $\pi(s_0) = s_1$ denotes the action that causes a transition from $s_0$ to $s_1$.
	
	As shown in \Cref{fig:example-gamma-1} and \Cref{fig:example-gamma-2}, the optimal static policy $\pi_{\gamma_1}^*$ of an agent who applies a constant discount factor $\gamma_1$ in the classical setting, is given by $\pi^{*}_{\gamma_1}(s_0) = s_0$, $ \pi^{*}_{\gamma_1}(s_1) = s_0 $ and $\pi^{*}_{\gamma_1}(s_2) = s_2$.
	For $\gamma_2$, the optimal static policy differs from  only in state $s_0$, namely $\pi^{*}_{\gamma_2}(s_0) = s_1$.
	
	We want to compute an SPE for the setting where the first player discounts with $\gamma_1$ and all subsequent players discount with $\gamma_2$. 
	Since the discount factor is constant from time step 1 onward, we can fix the policies of players $1,2,\dots$ to $\pi_{\gamma_2}^*$ to derive an SPE.
	Player $0$ knows all future players' policies and can influence future rewards merely by choosing the current action. Starting in $s_0$, the action given by policy $\pi_{\gamma_1}^*$ would be the one that transitions to itself. As player 0 is fairly short-sighted, knowing that player 1 will choose an action leading to reward $0$ in the subsequent time step, it prefers taking an immediate reward of 4 and transitioning into state $s_2$ to stay there forever and keep receiving rewards of 3.
	Hence, an SPE is given by $\ppi_{\text{SPE}} = (\pi', \pi^{*}_{\gamma_2}, \pi^{*}_{\gamma_2}, \ldots)$, where $ \pi'(s_0) = s_2 $; see \Cref{fig:example-spe}.

	\section{Existence of an SPE}
	
	We investigate the existence of an SPE. Indeed, the following result has already answered this question in affirmative.
	
	\begin{theorem}[\citealt{lattimore2014general}]
		An (exact) SPE always exists.
	\end{theorem}
	
	The result applies even without our assumption on the convergence of the discount function but it is obtained via a non-constructive approach: by reasoning about a sequence of policies that are optimal for the truncated versions of the problem, each with a longer horizon than their predecessors.
	Hence, the proof does not yield any algorithm or procedure for obtaining an SPE. We next provide a constructive proof. The proof also forms the basis for deriving a complexity upper bound of {\sc SPE-Start} as we will demonstrate later.
	
	We will in particular show that there exists an SPE that is {\em eventually constant}, i.e., there exists a time step after which all the subsequent players use the same static policy.
	A mild assumption is needed for our proof: $\vargamma$ converges to a value outside of a set of {\em degenerate points} defined below. 
	
	\begin{definition}[Degenerate point]
		A discount factor $\gamma$ is called a {\em degenerate point} if
		$\Pi_{\gamma}^*$ contains more than one {non-equivalent} policy (see \Cref{def:equi-policies}), i.e., $\abs{\Pi_{\gamma}^*/\sim} > 1$,
		where $\Pi_{\gamma}^*/\sim$ is the set of equivalence classes on $\Pi_{\gamma}^*$ under the equivalence relation $\sim$ defined in \Cref{def:equi-policies}, i.e., an element $\{\pi' \in \Pi : \pi' \sim \pi \} \in \Pi_{\gamma}^*/\sim$ contains all policies equivalent to $\pi$.
	\end{definition}
	
	In what follows, let
	\[
	\Gamma \coloneqq \left\{ \gamma \in [0,1) : \abs{\Pi_{\gamma}^*/\sim} > 1 \right\}
	\]
	be the set of degenerate points in $[0,1)$.
	The assumption above ensures that the players will eventually adopt the same behavior after some time step $T$, as if the subsequent process is a constant-discounting MDP. From that point on, the dynamic policy can be represented as a static one and we can use backward induction to derive policies for players in the previous time steps.
	More formally, the main result of this section is formulated as follows.
	
	\begin{restatable}{theorem}{existSPE}
		\label{thm:exist-SPE}
		Suppose that $\gamma^{*} \coloneqq \lim_{t \to \infty} \vargamma(t)$ exists and $\gamma^{*} \in [0,1] \setminus \Gamma$. Then there exists an SPE $\ppi$ that is eventually constant, i.e., there exists a number $T \in \mathbb{N}$ such that $\pi_t = \pi_T$ for all $t \ge T$.
	\end{restatable}
	
	\subsection{Proof of \Cref{thm:exist-SPE}}
	The key to the proof is to argue that $\Pi_{\vargamma(t)}^*$ is eventually constant after a certain time step $T$. 
	With this property, we can pick an arbitrary $\tilde{\pi} \in \Pi_{\vargamma(T)}^*$ and assign it to all the players $t \ge T$. 
	This forms an SPE for the subgame starting at $T$, and according to Observation~\ref{obs:V-constant-tail}, we can use $V_{g(t),T}^\pi$ as a basis 
	and use backward induction to construct $\pi_{T-1}, \pi_{T-2}, \dots, \pi_0$ as the optimal policies of players $T-1, T-2, \dots, 0$ with respect to $V_{g(T-1),T}^\pi, V_{g(T-2),T-1}^\pi, \dots, V_{g(0),1}^\pi$, respectively. 
	The approach is summarized in Algorithm~\ref{alg:construct-SPE}.
	
	\begin{algorithm}[t]
		\SetInd{0.3em}{0.6em}
		\caption{Constructing an SPE $\ppi = (\pi_t)_{t=0}^\infty$, given that $\pi_t = \tilde{\pi}$ for all $t\ge T$\label{alg:construct-SPE}}
		\SetKwInOut{Input}{Input}
		\SetKwInOut{Output}{Output}
		
		\BlankLine
		\Input{a static policy $\tilde{\pi} \in \Pi$, and a time step $T \in \mathbb{N}$}
		\Output{an SPE $\ppi = (\pi_t)_{t=0}^\infty$}
		
		\BlankLine
		\For{$t=T-1,T-2,\dots,0$}
		{
			Compute $V_{g(t)}^{\tilde{\pi}}$ defined according to \eqref{eq:v-func-const-gamma} and \eqref{eq:Q-func-const-gamma}\;
			
			
			$V_{t,T}(s) \leftarrow V_{g(t)}^{\tilde{\pi}}(s)$ for all $s \in S$;
			\tcp*{so $V_{t,T} = V_{g(t),T}^\ppi$ (Observation~\ref{obs:V-constant-tail})}
			
			\For
			{$i = T-1,T-2,\dots,t$}%
			{
				\For{each $s \in S, a \in A_s$}
				{
					$Q_{t, i}(s, a) \leftarrow R(s,a)  + \gamma \sum_{s'\in S}P(s, a, s')\cdot V_{t, i+1}(s')$; 
					
					$V_{t, i}(s) \leftarrow Q_{t,i}(s,\pi_{t+1}(s))$\tcp*{\hspace{-2mm}so $V_{t,i} {=} V_{g(t),i}^\ppi$ in \eqref{eq:v-func-vargamma}}
				}
			}
			
			\For
			{each $s \in S$}
			{
				$\pi_{t}(s) \leftarrow$ arbitrary action in $\argmax_{a \in A_s} Q_{t,t}(s,\pi_{t+1}(s))$;
			}
		}
	\end{algorithm}

	Now to show that $\Pi_{\vargamma(t)}^*$ is eventually constant, we argue that the set $\Gamma$ of degenerate points is finite (\Cref{lem:G-is-finite}).
	Since $\gamma^* \notin \Gamma$, there must be a neighbourhood of $\gamma^*$ in $\mathbb{R}$ which does not intersect $\Gamma$. After a certain time step, the tail of $\vargamma$ will be contained inside this neighbourhood, so \Cref{lmm:opt-Pi-interval-G} then implies that $\Pi_{\vargamma(t)}^*$ becomes constant after a finite number of time steps.
	
	\begin{restatable}{lemma}{Gfinite}
		\label{lem:G-is-finite}
		$\Gamma$ is a finite set.
	\end{restatable}
	
	\begin{proof}
		Define 
		\begin{align} 
			\label{eq:h-s}
			h^s_{\pi_1, \pi_2}(\gamma) \coloneqq V_{\gamma}^{\pi_1}(s) - V_{\gamma}^{\pi_2}(s)\thinspace.
		\end{align}
		By definition, for any $\gamma \in \Gamma$, there exist $\pi_1, \pi_2 \in \Pi_\gamma^*$ such that $\pi_1 \not\sim \pi_2$, which means that $h^s_{\pi_1, \pi_2}(\gamma) = 0$ for all $s \in S$.
		Hence, $|\Gamma|$ is bounded from above by the number of $\gamma$s such that $h_{\pi_1, \pi_2}^s (\gamma) = 0$ for some $s \in S$ and some $\pi_1, \pi_2 \in \Pi$ that are not equivalent.
		By \Cref{lem:h-function-as-sum-of-rational-functions},
		$h_{\pi_1, \pi_2}^s(\gamma) = \Psi(\gamma) / \Phi(\gamma)$, where both $\Psi(\gamma)$ and $\Phi(\gamma)$ are polynomial functions of $\gamma$ with finite degrees.
		Hence, the number of zeros of $h_{\pi_1, \pi_2}^s(\gamma)$ is finite.
	\end{proof}
	
	\begin{restatable}{lemma}{hSumOfRationals}
		\label{lem:h-function-as-sum-of-rational-functions}
		Let $\pi_1, \pi_2 \in \Pi$ be two policies and $\gamma \in [0,1)$. 
		The function $h^s_{\pi_1, \pi_2}(\gamma)$ can be written as
		\begin{equation}
			\label{eq:poly-form}
			h^s_{\pi_1, \pi_2}(\gamma) = \Psi(\gamma)/\Phi(\gamma)\thinspace,
		\end{equation}
		where $\Psi$ and $\Phi$ are polynomials of $\gamma$ with finite degrees.
	\end{restatable} 
	\noindent This and subsequent omitted proofs can be found in the appendix.

	\begin{restatable}{lemma}{optPiIntervalG}
		\label{lmm:opt-Pi-interval-G}
		For any interval $I \subseteq [0,1)$ such that $I \cap \Gamma = \emptyset$, we have $\Pi_{\gamma}^* = \Pi_{\gamma'}^*$ for any $\gamma, \gamma' \in I$.
	\end{restatable}
	
	\begin{proof}
		Without loss of generality, suppose that $\gamma < \gamma'$ and, for the sake of contradiction, $\Pi_{\gamma}^* \neq \Pi_{\gamma'}^*$. 
		The fact that $\gamma, \gamma' \notin \Gamma$ directly implies that $\Pi_{\gamma}^* \cap \Pi_{\gamma'}^* = \emptyset$ as both sets contain only equivalent policies. 
		Pick arbitrary $\pi \in \Pi_{\gamma}^*$ and $\pi' \in \Pi_{\gamma'}^*$. 
		We have $h_{\pi, \pi'}(\gamma) > 0$ and $h_{\pi, \pi'}(\gamma') < 0$.
		As $h_{\pi, \pi'}$ is a continuous function, there exists a $\tilde{\gamma} \in (\gamma, \gamma') \subseteq I$ such that $h_{\pi, \pi'}(\gamma') = 0$, 
		which implies $\abs{\Pi_{\tilde{\gamma}}^*/\sim} > 1$ and hence $\tilde{\gamma} \in \Gamma$. This contradicts the assumption that $I \cap \Gamma = \emptyset$.
	\end{proof}

	\section{Complexity of {\sc SPE-Start}}
	\label{section:computation}
	Consider using Algorithm~\ref{alg:construct-SPE} to construct an SPE. It requires specifying $T$ in the input which we have not yet described how to obtain.
	Indeed, this replies on the specific format of $g$.
	In addition to the computational cost of obtaining $T$, Algorithm~\ref{alg:construct-SPE} includes $O(T^2 \cdot |S|)$ iterations, so the overall time complexity also depends on the magnitude of $T$.
	The latter cost prevents the algorithm from being efficient if $T$ is exponential in the size of the problem, so the question is whether there are better algorithms that solve {\sc SPE-Start} without going through all the iterations.
	It turns out that this is in general not possible: as we will show next, even for discount functions that admit efficient computation of $T$, computing an SPE can be EXPTIME-hard.
	
	\subsection{EXPTIME-Hardness}
	
	We show that {\sc SPE-Start} is EXPTIME-hard even when $\vargamma: \mathbb{N} \to [0,1)$ is a {\em down-step} function defined as follows:
	\begin{equation}
		\label{eq:down-step}
		\vargamma(t) \coloneqq
		\begin{cases}
			\gamma & \text{ if } t \le T \\
			0 & \text{ otherwise}\thinspace,
		\end{cases}
	\end{equation}
	where $\gamma \in (0,1)$ and $T \in \mathbb{N}$ is encoded in binary. Arguably this is one of the simplest forms of time-varying discounting, and $\vargamma$ can be encoded as $(\gamma,T)$.
	
	Note that \eqref{eq:down-step} does not define a finite horizon MDP. 
	Instead, it defines a game where eventually all the players from time step $T$ onward exhibit a discount factor of 0.
	We will show that {\sc SPE-Start} is EXPTIME-hard even when the discount function is restricted to this simple form. 
	The proof uses a reduction from the following problem, termed {\sc ValIt} (value iteration), which is known to be EXPTIME-complete \citep{shirmohammadi2019complexity}.
	
	\begin{definition}[{\sc ValIt}]
		An instance of {\sc ValIt}, given by $(\calM,a^\dag,T)$, consists of an MDP $\calM = (S, A, R, P, \sstart, \gamma)$ with constant discount factor $\gamma > 0$, an action $a^\dag \in A$, and finite time horizon $T \in \mathbb{N}$ encoded in binary.
		It is a yes-instance if there exists a dynamic policy $\ppi$ such that $\pi_0(\sstart) = a^\dag$ and $\pi_t(s) \in \argmax_{a \in A_{s}} Q_t(s, a)$ for all $t = 0,\dots,T-1$ and $s \in S$,
		where 
		\begin{align}
			Q_t(s,a) &\coloneqq  R(s,a) + \gamma \sum_{s'\in S}P(s,a, s')\cdot V_{t+1}(s')\thinspace,  \label{eq:valit-Q} \\
			V_t(s) &\coloneqq \max_{a\in A_s} Q_t(s,a)\thinspace,
			\label{eq:valit}
		\end{align}
		and $Q_{T}(s,a) \equiv 0$.
		Otherwise, it is a no-instance.
	\end{definition}
	
	The $V_t$ functions in the above definition are akin to the value functions defined in \eqref{eq:V-pi-gamma} but with a time-dependency.
	Using {\sc ValIt}, we prove the following result.
	
	\begin{restatable}{theorem}{EXPTIMEhardnessDownStep}
		\label{thm:SPE-start-is-EXPTIME-hard-for-down-step-functions}
		{\sc SPE-Start} is EXPTIME-hard even when the discount function is a down-step function.
	\end{restatable}
	
	\begin{proof}[Proof sketch]
		We reduce {\sc ValIt} to {\sc SPE-Start}.
		The main idea of the reduction is to construct an {\sc SPE-Start} instance where all players $t > T$ will stick to the same static policies regardless of policies chosen by the preceding players. 
		\Cref{figure:reduction-to-one-discount-factor} illustrates the MDP in the {\sc SPE-Start} instance. A chain consisting of two states $s^*$ and $s^{**}$ is appended to every state $s$ in the {\sc ValIt} instance. The high reward at $a^*$ ensures that $a^*$ is the {\em dominant} action for player $T$, who has $\vargamma(T) = 0$ and only cares about the immediate reward; whereas the high penalty at $a^{**}$ ensures that $a^*$ is a {\em dominated} action for all players $t = 0,\dots,T$, who have $\vargamma(t) = \gamma$. 
		Hence, for every player $t \leq T$ in the {\sc SPE-Start} instance, the process is equivalent to an MDP with time horizon $T+1$.
		The procedure to derive $\pi_T, \pi_{T-1}, \dots, \pi_0$ in an SPE using backward induction is the same as computing the value functions of the {\sc ValIt} instance. Every SPE is then associated to an optimal policy of {\sc ValIt}.
	\end{proof}
	
	We remark that the binary encoding of $T$ plays a crucial role in the EXPTIME-hardness of {\sc SPE-Start}. 
	Indeed, if $T$ is encoded in unary or is a constant, the hardness will disappear.
	In general, an efficient algorithm for computing an SPE is possible but requires the assumption of $g$ converging fast enough to an interval between two consecutive numbers in $\Gamma$.
	To ease part of the intricacies introduced by the requirement, a practical approach which we will present next is by considering the approximate notion of the SPE, the $\epsilon$-SPE.

	\paragraph{Remark} 
	The down-step function we defined in \eqref{eq:down-step} may appear to only be representative of decreasing functions. However, the EXPTIME-hardness remains if we consider simple increasing functions $\vargamma$ such that $\vargamma(t) = \gamma_1$ if $t \le T$, $\vargamma(t) = \gamma_2$ if $t > T$, and $\gamma_1 < \gamma_2$; our proof can be easily extended to such functions.
	
	\begin{figure}[t]
		\centering
		\includegraphics[width=0.7\linewidth]{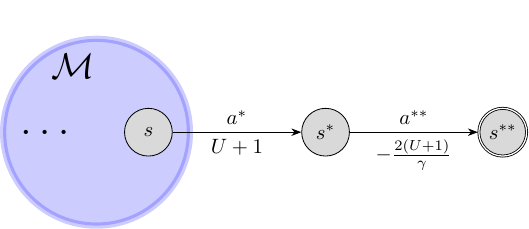}
		\caption{Reduction from $\valit$ to {\sc SPE-Start}. The blue disk represents the original MDP $\calM$ in the $\valit$ instance, and the outer nodes indicate how to extend $\calM$ to an MDP for the {\sc SPE-Start} instance, where the discount rate is fixed to the original discount rate $\gamma$ in the first $T$ time steps and set to $0$ afterwards. Labels above edges are action names and labels below are rewards. 
		}
		\label{figure:reduction-to-one-discount-factor}
	\end{figure}
	
	\section{Approximate SPEs}
	
	The $\epsilon$-SPE, defined below, assumes that the players are reluctant to deviate as long as the potential improvement is smaller than some $\epsilon > 0$. 
	
	\begin{definition}[$\epsilon$-SPE]
		A dynamic policy $\ppi = (\pi_t)_{t=0}^\infty$ forms an $\epsilon$-SPE if for all $t \in \mathbb{N}$ it holds that for all $s \in S$: $u_t(\ppi|s) \ge u_t(\ppi'|s) - \epsilon$ for all $\ppi'= (\pi'_t)_{t=0}^\infty$ such that $\pi'_i = \pi_i$ for all $i \in \mathbb{N} \setminus \{t\}$.
	\end{definition}
	
	The notion allows us to relax the assumption that $\vargamma$ converges to a point outside of $\Gamma$ and allows us to derive an upper bound of the computational complexity, too. 
	Indeed, the existence of an $\epsilon$-SPE, in particular an {\em eventually constant} one, does not require this assumption.
	Instead, we use the following continuity argument.
	
	\begin{restatable}{lemma}{epsSPEbound}
		\label{lmm:eSPE-bound}
		Suppose that all rewards are bounded by $M$. Then for any discount factors $\gamma, \tgamma \in [0,1)$ and static policy $\pi \in \Pi$, we have the following bound for all $s \in S$:
		\[
		\abs{V_{\gamma}^{\pi}(s) - V_{\tgamma}^{\pi}(s)}
		\le \frac{2 M \cdot |S| \cdot \abs{\gamma - \tgamma}}{\left(1-\max\{\gamma, \tgamma\}\right)^3  \cdot \left(1-\min\{\gamma, \tgamma\} \right)}\thinspace.
		\]
	\end{restatable}
	
	\begin{restatable}{theorem}{existEpsSPE}
		\label{thm:exist-eps-spe}
		Suppose $\gamma^{*} \coloneqq \lim_{t \to \infty} \vargamma(t)$ exists and $\gamma^{*} \in [0,1]$. For any $\epsilon > 0$, there exists an $\epsilon$-SPE $\ppi$ that is {\em eventually constant}, i.e., there exists a $T \in \mathbb{N}$ such that $\pi_t = \pi_T$ for all $t \ge T$.
	\end{restatable}

	\subsection{Computing an $\epsilon$-SPE}
	
	The $\epsilon$ slackness introduced by $\epsilon$-SPE appears to suggest that it suffices to consider a finite time horizon: a player can cut off the time horizon up to a certain (finite) future time step, beyond which the sum of the discounted rewards is sufficiently small to be ignored. This is nevertheless not the case. Even if the horizon is cut off, there are still infinitely many players in the game and each player's payoff is influenced by the subsequent players before the cutting off point. Hence, cutting off the time horizon does not reduce our consideration to a finite number of time steps.

	To compute an $\epsilon$-SPE, we use the continuity argument in \Cref{lmm:eSPE-bound}. 
	If we can pin down a time step $t$ after which the tail of $\vargamma$ is contained in a sufficiently small interval, we can use $\vargamma(t)$ to compute an SPE for the subgame as if $\vargamma$ is constant after $t$. This approximates an actual SPE provided that the tail of $\vargamma(t)$ is sufficiently small.
	Hence, the time complexity depends on the rate at which $\vargamma$ converges.
	In accordance with our existence proof, let $d$ be such that
	\[
	\min_{\gamma \in \Gamma}\ \left| \lim_{t\to \infty}g(t) - \gamma \right| \ge d\thinspace.
	\]
	To derive a general result, we also assume that there is an oracle $\mathcal{A}$ that, for any given $\delta > 0$, computes a time step $T$ such that $|\vargamma(t) - \vargamma(T)| \le \delta$ for all $t \ge T$. 
	More specifically, we introduce the following notion called $(\alpha,\beta)$-convergence for the discount function. 
	
	\begin{definition}[$(\alpha,\beta)$-convergence]
		Let $\alpha: \mathbb{R}^2 \to \mathbb{R}$ and $\beta : \mathbb{R}^2 \to \mathbb{R}$.
		A class $\calG$ of discount functions is {\em $(\alpha,\beta)$-convergent}
		if there is an oracle $\mathcal{A}$ such that: 
		for any $\vargamma \in \calG$ and any $\delta > 0$ with bit-size $d$, $\mathcal{A}$ computes an integer $T$ in time $\alpha(|\vargamma|, d)$ such that $\abs{\vargamma(t) - \vargamma(T)} \le \delta$ for all $t \ge T$, and $T \le \beta(|\vargamma|,d)$, where $|\vargamma|$ denotes the bit-size of the representation of $\vargamma$.
	\end{definition}
	
	For example, the class of down-step functions, defined in \eqref{eq:down-step} and encoded as $(\gamma, T)$ (in binary), is $(\alpha,\beta)$-convergent with $\alpha(x,y) = x$ and $\beta(x, y) = O(2^{x})$. Our next lemma provides a lower bound on the distance between any two points in the set $\Gamma$.

	\begin{restatable}{theorem}{compTimeAlphaBetaConvergent}
		\label{thm:eps-SPE-run-time}
		Suppose that $\vargamma$ is $(\alpha,\beta)$-convergent and $\lim_{t\to \infty} \vargamma(t) < 1 - c$ for a known constant $c$.
		Then an $\epsilon$-SPE can be computed in time 
		\[
		\alpha(|\vargamma|, d) + \poly(|A|,|S|) \cdot \left(\beta(|\vargamma|, d) \right)^2\thinspace,
		\]
		where $d = \log({M |S|}/{\epsilon}) + o(1)$.
	\end{restatable}

	\begin{proof}
		We use Algorithm~\ref{alg:computing-eSPE} to compute an $\epsilon$-SPE. To see that it correctly computes an $\epsilon$-SPE, it suffices to argue that $(\pi_T,\pi_{T+1}, \dots)$ form an $\epsilon$-SPE for the subgame after $T$.
		
		Indeed, for any player $t\ge T$, we have $\abs{ \vargamma(t) - \vargamma(T)} \le D \le \frac{c^4 \cdot \epsilon}{2M|S|}$.
		Hence, according to \Cref{lmm:eSPE-bound}, we have 
		$\abs{V_{\vargamma(t)}^{\pi}(s) - V_{\vargamma(T)}^{\pi}(s)}
		\le \epsilon/2$ for any static policy $\pi$ and $s \in S$.
		Let $\pi \in \Pi_{g(t)}^*$.
		We have 
		\begin{align*}
			V_{\vargamma(t)}^{\pi}(s) - V_{\vargamma(t)}^{\tilde{\pi}}(s) 
			\le
			V_{\vargamma(T)}^{\pi}(s)  - V_{\vargamma(T)}^{\tilde{\pi}}(s) + 
			\abs{V_{\vargamma(t)}^{\pi}(s) - 
				V_{\vargamma(T)}^{\pi}(s)} +
			\abs{V_{\vargamma(t)}^{\tilde{\pi}}(s) - V_{\vargamma(T)}^{\tilde{\pi}}(s)} 
			\le \epsilon\thinspace,
		\end{align*}
		where $\tilde{\pi} \in \Pi_{g(T)}^*$ as in Algorithm~\ref{alg:computing-eSPE}.
		Moreover, since the optimal static policy is at least as good as any dynamic policy for player $t$. 
		This means that for any strategy profile $\ppi'$ resulting from a deviation of player $t$,
		\[u_t(\ppi'|s) - u_t(\pi_T,\pi_{T+1}, \dots|s) \le V_{\vargamma(t)}^{\pi}(s) - V_{\vargamma(t)}^{\tilde{\pi}}(s) = \epsilon\thinspace. \]
		Hence, Algorithm~\ref{alg:computing-eSPE} generates an $\epsilon$-SPE.
		
		To see the time complexity of the algorithm, note that it takes time $\alpha(|\vargamma|, \log D)$ \linebreak to run $\mathcal{A}$.
		In addition to that, the time it takes to run Algorithm~\ref{alg:construct-SPE} is bounded by \linebreak $\left(\beta(|\vargamma|, d) \right)^2 \cdot \poly(|A|,|S|)$.
	\end{proof}

	\begin{algorithm}[t]
		\caption{Computing an $\epsilon$-SPE}\label{alg:computing-eSPE}
		\SetKwInOut{Input}{Input}
		\SetKwInOut{Output}{Output}
		
		\BlankLine
		
		\Input{$\epsilon > 0$}
		\Output{an $\epsilon$-SPE $\ppi = (\pi_t)_{t=0}^\infty$}
		
		\BlankLine
		
		$D \leftarrow c^4 \cdot \min \{ \epsilon / 4M |S|,\ c \}$;
		
		$T \leftarrow \mathcal{A}(D)$;
		
		$\tilde{\pi} \leftarrow$ an arbitrary policy in $\Pi_{g(T)}^*$;
		
		$\pi_t \leftarrow \tilde{\pi}$, for all $t = T, T+1,\dots$;
		
		\mbox{Run Algorithm~\ref{alg:construct-SPE} on input $\tilde{\pi}$ to construct $\pi_0,\dots,\pi_{T-1}$;}
	\end{algorithm}
	
	We remark that \Cref{thm:eps-SPE-run-time} only requires the mild assumption of a known constant gap between $1$ and the limit point of $g$. 
	If $c$ is unknown or the gap cannot be bounded by a constant, an $\epsilon$-SPE can be computed via a more sophisticated algorithm with a higher time complexity. We provide this algorithm in the full version of this paper for theoretical interest.
	
	Via \Cref{thm:eps-SPE-run-time}, an exponential upper bound of the complexity of computing an $\epsilon$-SPE 
	can be derived when $\vargamma$ is the down-step function defined in \eqref{eq:down-step} (for which $\beta(|\vargamma|,d) = 2^{O(|\vargamma|)}$). This does not require any assumption on the convergence of $g$ with respect to $\Gamma$.  
	Better bounds can be derived if $\vargamma$ converges faster, e.g., $\beta(|\vargamma|,d) = d$ or even $2^{O(d)}$, or when $g$ is not a variable of the model.

	\section{Conclusion}
	\label{section:conclusion}
	
	We study a model of infinite-horizon MDPs with time-varying discounting.
	Our model seizes the idea of geometric discounting, but with time-varying discount factors, and it allows for a game-theoretic interpretation. We study the SPE of the underlying game. Results on the existence and computation of an exact or an $\epsilon$-SPE are presented.
	Future work may consider other types of discount functions, such as those described by \citet{lattimore2014general}.

	\newpage
	\section*{Acknowledgements}
	The authors would like to thank the anonymous reviewers for their insightful comments.
	This project was supported by DFG project 389792660 TRR 248--CPEC.
	Part of the work was done when Jiarui Gan was a postdoc at MPI-SWS, when he was supported by the European Research Council (ERC) under the European Union's Horizon 2020 research and innovation programme (grant agreement No. 945719).
	Part of the work was done when Annika Hennes was an intern at MPI-SWS. Further, she was partially funded by the Deutsche Forschungsgemeinschaft (DFG, German Research Foundation) – project number 456558332. 
	Goran Radanovic's research was, in part, funded by the Deutsche Forschungsgemeinschaft (DFG, German Research Foundation) – project number 467367360.

	\bibliographystyle{plainnat}
	\bibliography{ms}
	
	\clearpage
	
	\appendix

\section{Existence of an SPE}

\hSumOfRationals*

\begin{proof}
	Fix a deterministic policy $\pi$ and a state $s^\star \in S$.
	For each state $s \in S$ and each $i\in \mathbb{N}$, let $K_s^i$ be the number of steps between the $(i-1)$-th and the $i$-th times the state is in $s$, when one starts from $s^\star$ and takes actions according to policy $\pi$; hence, $K_s^i$ is a random variable.
	These variables are independent and identical for $i \ge 2$.
	This way, we can write the cumulative reward as 
	\begin{equation}
		\label{eq:V-f}
		V_{\gamma}^{\pi}(s^\star) = 
		\sum_{s \in S} f_\pi^s(\gamma)\thinspace,
	\end{equation}
	where 
	\begin{align}
		f_\pi^s(\gamma)
		&\coloneqq \expct{\gamma^{K_s^1} \bigl(R(s,\pi(s)) \nonumber + \gamma^{K_s^2}\bigl(R(s, \pi(s)) \nonumber + \gamma^{K_s^3}\bigl(R(s, \pi(s)) + \dots \bigr) \bigr) \bigr)} \nonumber\\
		&= \expct{\gamma^{K_s^1} \cdot R(s, \pi(s)) } + \expct{\gamma^{K_s^1}} \cdot F
		\label{eq:f-pi-s-gamma}
	\end{align}
	with 
	\[F \coloneqq \mathbb{E} \left [ \gamma^{K_s^2}\left( R(s, \pi(s)) + \gamma^{K_s^3}\left( R(s, \pi(s)) + \dots \right) \right) \right]\thinspace.\]
	As shown above, we let the last term be $F$.
	Since $K_s^2, K_s^3, \dots$ are i.i.d., we then have 
	\[
	F = \expct{ \gamma^{K_s^2}} \cdot R(s, \pi(s)) + \expct{ \gamma^{K_s^2}} \cdot F\thinspace,
	\]
	which means that
	\[
	F = \frac{\expct{\gamma^{K_s^2} }}{1 - \expct{\gamma^{K_s^2} }} \cdot R(s, \pi(s))\thinspace.
	\]
	Substituting the above into \eqref{eq:f-pi-s-gamma} gives
	\begin{equation}
		\label{eq:f-pi-s-gamma-2}
		f_\pi^s(\gamma) 
		= \E\bigl[ \gamma^{K_s^1} \bigr]\cdot R(s, \pi(s))\cdot\frac{1}{1-\E[\gamma^{K_s^2}]}\thinspace.
	\end{equation}
	It remains to do some analysis on $\E[ \gamma^{K_s^1} ]$ and $\E[ \gamma^{K_s^2} ]$.
	Let $P_{s'}^s(i)$ denote the probability that we reach state $s'$ in $i$ steps when we start from state $s$ and take actions according to $\pi$, i.e.,
	\[ 
	P_s^{s'}\!(i) \coloneqq 
	\Pr\left(s_i = s' \text{ and } s_t \neq s' \text{ for all } t < i \mid s_0 = s, \pi \right)\thinspace. 
	\] 
	Let $p_s = \left[P \left(s, \pi(s), s'\right) \right]_{s' \in S \setminus \{s\}}^\top$ be a (row) vector, where each element $P(s, \pi(s), s')$ is the probability of the state transitioning from $s$ to $s'$ after action $\pi(s)$ is taken. Similarly, let $ q^s = \left[P \left(s', \pi(s'), s \right) \right]_{s' \in S \setminus \{s\}} $ be a (column) vector and $M= \left[P\left( s', \pi(s'), s''\right)\right]_{s', s'' \in S \setminus \{s\}}$ be a matrix.
	We can now write
	\[ 
	\expct{\gamma^{K_s^2}} = \sum_{i=1}^{\infty} P_s^s(i) \cdot \gamma^i\thinspace. 
	\]
	Moreover, for $i \geq 2 $, we have $ P_s^s(i) = p_s\cdot M^{i-2}\cdot q^s $, and hence
	\begin{align*}
		\sum_{i=2}^{\infty}P_s^s(i) \cdot \gamma^i &= \sum_{i=2}^{\infty} \left(p_s\cdot M^{i-2}\cdot q^s \cdot \gamma^i \right) \\
		&= p_s\cdot\biggl(\sum_{i=2}^{\infty}M^{i-2}\gamma^i\biggr)\cdot q^s \\
		&= p_s\cdot \gamma^2\cdot\biggl(\sum_{i=2}^{\infty}(M\gamma)^{i-2}\biggr) \cdot q^s\thinspace. 
	\end{align*} 
	Letting $ B(\gamma) \coloneqq \sum_{i=2}^{\infty}(M\gamma)^{i-2}$ then gives
	\[ 
	B(\gamma) = I + (M\gamma)\sum_{i=2}^{\infty}(M\gamma)^{i-2} = I + (M\gamma) \cdot B(\gamma)\thinspace, 
	\]
	which is equivalent to
	\[ 
	B(\gamma) = \left(I - M\gamma \right)^{-1}\thinspace. 
	\]
	Indeed, since all entries in $M$ are at most 1 and $\gamma < 1$, 
	we have $ \lim_{i\rightarrow \infty} (M\gamma)^{i} = 0 $. This implies that the Neumann series $ \sum_{i=0}^{\infty} (M\gamma)^{i} $ converges to $(I - M\gamma)^{-1} $, which shows that the inverse of $ I - M\gamma $ exists.
	The Cayley-Hamilton theorem \citep{hamilton1853lectures} implies that the entries of $\left(I - \gamma M \right)^{-1}$ are rational functions where the numerator and denominator are polynomials in $ \gamma $ of degree at most the size of $M$, i.e., $|S| - 1$.
	
	To sum up, we have
	\begin{align}
		\E\bigl[ \gamma^{K_s^2} \bigr] 
		&= P_s^s(1) \cdot \gamma + p_s\cdot\gamma^2\cdot B(\gamma) \cdot q^s \nonumber \\
		&= P\left(s,\pi(s),s \right) \cdot \gamma + p_s\cdot\gamma^2\cdot B(\gamma) \cdot q^s\thinspace, \label{eq:E-gamma-K-s-2}
	\end{align}
	which is a function of the form
	\begin{align}
		\sum_{i=1}^n \sum_{j=1}^n \frac{\psi^i(\gamma)}{\phi^j(\gamma)}\thinspace,
	\end{align}
	where $\psi^i(\gamma)$ and $\phi^j(\gamma)$ are polynomial functions of $\gamma$ with finite degrees.
	A similar reasoning can be applied to $ \E\bigl[ \gamma^{K_s^1} \bigr] $, if we consider 
	$\E\bigl[ \gamma^{K_s^1} \bigr] = \sum_{i=1}^{\infty}P_{\!s^\star}^s\!(i) \cdot \gamma^i$.
	We have
	\begin{align*}
		\E\bigl[ \gamma^{K_s^1} \bigr] 
		&= P_{\!s^\star}^s\!(1) \cdot \gamma + p_{s^\star}\cdot\gamma^2\cdot B(\gamma)\cdot q^s \\
		&= P \left(s, \pi(s), s^\star \right) \cdot \gamma + p_{s^\star}\cdot\gamma^2\cdot B(\gamma)\cdot q^s\thinspace,
	\end{align*}
	which is again a rational function with finite degrees.
	Hence, for all $s\in S$ the functions $ f_{\pi}^s $ and hence also their sum $V_\gamma^\pi(s^\star)$ are rational functions of 
	the form shown in \eqref{eq:poly-form}.
	It follows that for any choice of $\pi_1, \pi_2$, the function $h_{\pi_1, \pi_2}^{s^\star}$, according to its definition in \eqref{eq:h-s}, is also 
	of the same form.
\end{proof}

\section{EXPTIME-Hardness of {\sc SPE-Start}}

\EXPTIMEhardnessDownStep*
\begin{proof}
	We reduce an arbitrary instance $(\calM,a^\dag,T)$ of {\sc ValIt} to the following instance of {\sc SPE-Start}. Since {\sc ValIt} is EXPTIME-complete~\citep{shirmohammadi2019complexity}, the result follows.     
	
	Let $\widetilde{\calM} \coloneqq (\widetilde{S}, \widetilde{A}, \widetilde{R}, \widetilde{P}, s_0, \vargamma(\cdot))$ be an MDP with $\widetilde{S} \coloneqq S \cup \{ s^*, s^{**} \} $, $\widetilde{A} \coloneqq A \cup \{a^*, a^{**} \}$,
	i.e., we add two additional states and two additional actions.
	As illustrated in Figure~\ref{figure:reduction-to-one-discount-factor},
	taking action $a^*$ in any state $s \in S$ results in a deterministic state transition to $s^*$.
	Moreover, in $s^*$ the only action available is $a^{**}$, which takes the state to the terminal state $s^{**}$ with probability one.
	The transition dynamics remain the same with respect to all other state-action pairs $(s,a) \in S \times A$.
	The reward function fulfils $\widetilde{R}\colon \widetilde{S} \times \widetilde{A} \rightarrow \mathbb{R}$, such that
	\begin{align*}
		\widetilde{R}(s, a) = 
		\begin{cases}
			R(s,a) & \text{ if } (s,a) \in S \times A \\
			U + 1 & \text{ if } a = a^* \\
			- 2(U + 1) /\gamma & \text{ if } a = a^{**}
		\end{cases}
	\end{align*}
	where $U \coloneqq \max_{(s,a)\in S\times A} \frac{1}{1 - \gamma} \cdot \abs{R(s,a)}$.
	Finally, we let the discount function $\vargamma$ be a down-step function as defined in \eqref{eq:down-step}, where $\gamma$ and $T$ are taken directly from the {\sc ValIt} instance.
	
	We will next argue that an optimal policy for the {\sc ValIt} instance $(\calM, a^\dag, T)$ is equivalent to an SPE of the game defined by $\widetilde{\calM}$ and $\vargamma$.
	Specifically, we can map every {\sc ValIt} solution $\ppi$ to an {\sc SPE-Start} solution $\widetilde{\ppi}$ in a way such that $\widetilde{\pi}_t(s) = \pi_t(s)$ for all $t \le T$ and $s \in S$.
	We argue that
	$\ppi = (\pi_t)_{t=0}^T$ is an optimal policy for the {\sc ValIt} instance if and only if $\widetilde{\ppi} = (\widetilde{\pi}_t)_{t=0}^\infty$ is an SPE.
	Notably, we have $\widetilde{\pi}_0(s_0) = \pi_0(s_0)$, so the equivalence implies that an answer to the {\sc SPE-Start} instance also answers {\sc ValIt} and hence, {\sc ValIt} is reduced to {\sc SPE-Start}.
	
	To prove the above statement, we first note that in an SPE, players $0,1,\dots,T$ will never choose action $a^*$.
	Indeed, suppose player $t \le T$ is the first one who takes action $a^*$, then the process will terminate in two steps at $s^{**}$, and the sum of discounted rewards for player $t$ (for whom we have $\vargamma(t) =\gamma$) in the subsequent two steps will be 
	\[
	\widetilde{R}(s, a^*) + \gamma \cdot \widetilde{R}(s^*, a^{**}) = - U - 1\thinspace,
	\]
	where $s \in S$ since this player is the first one to take action $a^*$.
	In contrast, if this player switches to any action $a \neq a^*$, the environment will remain in a state $s' \in S$ for at least one more step, whereby the subsequent  sum of discounted rewards is at least
	\begin{equation} \label{eq:reduction-sum-of-rewards-for-staying-in-original-mdp}
		- \sum_{i=0}^{k-1} \gamma^i \cdot  R^* + \gamma^{k} \cdot \widetilde{R}(s'', a^*) + \gamma^{k+1} \widetilde{R}(s^*, a^{**})\thinspace,
	\end{equation}
	where $R^* \coloneqq \max_{(s,a) \in S\times A} \abs{R(s,a)} = U \cdot (1 - \gamma)$, $k \ge 1$ is the number of time steps where the state remains in $S$, and $s'' \in S$. 
	We have 
	\begin{align*}
		\eqref{eq:reduction-sum-of-rewards-for-staying-in-original-mdp} &=
		- \sum_{i=0}^{k-1} \gamma^i \cdot  R^* + \gamma^{k} \cdot (-U - 1) \\
		&=
		- U - \gamma^{k} > - U - 1\thinspace,
	\end{align*}
	so player $t$ obtains a strictly higher utility than when playing $a^*$.
	Hence, the state of the environment will remain in $S$ when it is for player $T+1$ to play an action.
	Since $\vargamma(T+1) = 0$, player $T+1$ will surely play action $a^*$, whereby it obtains an immediate reward of $U + 1$, which is strictly larger than what can be obtained by taking any other actions according to the definition of $\widetilde{R}$.
	Hence, the game will proceed in exactly the same way after time step $T$, irrespective of the actions performed by players $0,\dots,T$. 
	From the perspective of each player $t = 0,1,\dots,T$, their actions only change the game up to time step $T$, so $Q_{g(t),T}^\ppi(s,a)$ is a constant independent of $\ppi$, $s$, and $a$.
	This means that an SPE $\ppi$ in this particular game, which can be characterized by \eqref{eq:SPE-v-func}--\eqref{eq:Q-func-vargamma} with constant $Q_{g(t),T}^\ppi$, is effectively equivalent to an optimal policy defined through \eqref{eq:valit-Q} and \eqref{eq:valit} (Note that \eqref{eq:valit-Q} and \eqref{eq:valit} define the same optimal policies if we start by letting $Q_T(s,a)$ be any arbitrary constant).
\end{proof}



\section{Computation of SPEs and $\epsilon$-SPEs}

\epsSPEbound*

\begin{proof}
	Without loss of generality, we can assume $\gamma > \tgamma$. We will write $\delta = \gamma - \tgamma$. Note that this implies $\delta \le \gamma$.
	Since $V_{\gamma}^{\pi}(s) = \sum_{s'\in S} f^{s'}_\pi(\gamma)$, it suffices to bound $\abs{f^s_\pi(\gamma) - f^s_\pi(\tilde{\gamma})}$ for every $s \in S$.
	We have
	\begin{align*}
		\abs{f^s_\pi(\gamma) - f^s_\pi(\tilde{\gamma})}
		&\le  \biggl|\frac{\E[\gamma^{K^1_s}] }{1 - \E[\gamma^{K^2_s}] }- \frac{\E[\tgamma^{K^1_s} ]}{1 - \E[\tgamma^{K^2_s}] } \biggr| \cdot \abs{R(s,\pi(s))}\\
		&\le M \cdot \biggl(\frac{(1 - \E[\gamma^{K^2_s} ]) \cdot  \bigl|\E[\gamma^{K^1_s}]-\E[\tgamma^{K^1_s}] \bigr|}{(1 - \E[\gamma^{K^2_s}])(1 - \E[\tgamma^{K^2_s}]  )} +\frac{\E[\gamma^{K^1_s}] \cdot  \bigl| \E[\gamma^{K^2_s}] - \E[\tgamma^{K^2_s}]\bigr| }{(1 - \E[\gamma^{K^2_s}])(1 - \E[\tgamma^{K^2_s}] )}\biggr) \\
		&\le M \cdot \frac{\bigl|\E[\gamma^{K^1_s}]-\E[\tgamma^{K^1_s}]\bigr|}{(1 - \E[\tgamma^{K^2_s}]  )} + M \cdot  \frac{\E[\gamma^{K^1_s}]}{(1 - \E[\gamma^{K^2_s}] )} \cdot  \frac{\bigl| \E[\gamma^{K^2_s}] - \E[\tgamma^{K^2_s}] \bigr|}{(1 - \E[\tgamma^{K^2_s}] )}\thinspace.
	\end{align*}
	Now consider the first term in the summation:
	\begin{equation}
		\frac{\abs*{\expct{\gamma^{K^1_s}}-\expct{\tgamma^{K^1_s}} }}{\left(1 - \expct{\tgamma^{K^2_s}}  \right)}
		= \abs*{\expct{\gamma^{K^1_s} - \tgamma^{K^1_s}} } \cdot  \left(1 - \expct{\tgamma^{K^2_s}}  \right)^{-1}\thinspace. \label{eq:temp-stochastic-1} 
	\end{equation}
	We now use the following inequality. For any $l$, we have
	\[
	\gamma^l - \tgamma^l = \gamma^l - (\gamma - \delta)^l = \gamma^l \left(1 - \left(1 - \frac{\delta}{\gamma}\right)^l \right) \le \gamma^l \cdot \frac{\delta l}{\gamma}\thinspace.
	\]
	The last inequality uses $(1-x)^n \ge 1 - nx$ for $x \le 1$ and $\delta / \gamma \le 1$. 
	Now we have the following bound on the first term in \Cref{eq:temp-stochastic-1}.
	\begin{align*}
		\expct{\gamma^{K^1_s} - \tgamma^{K^1_s}} &= \sum_{d = 1}^{\infty} \Pr\left(K^1_s = d\right) \cdot \left( \gamma^d - \tgamma^d\right) \\
		&\le \delta \sum_{d=1}^{\infty} d \gamma^{d-1} = \frac{\delta}{(1-\gamma)^2}\thinspace.
	\end{align*}
	Substituting the above bound in \Cref{eq:temp-stochastic-1} and using $\expct{\tgamma^{K^2_s}} \le \tgamma$, we get the following bound:
	\begin{equation*}
		\frac{\abso{\expct{\gamma^{K^1_s}}-\expct{\tgamma^{K^1_s}} }}{\left(1 - \expct{\tgamma^{K^2_s}}  \right)} \le \frac{\delta}{(1-\gamma)^2(1-\tgamma)}\thinspace.
	\end{equation*}
	We can establish a similar bound on 
	\[
	\frac{\abs*{ \expct{\gamma^{K^2_s}} - \expct{\tgamma^{K^2_s}} }}{ \left(1 - \expct{\tgamma^{K^2_s}}  \right)}\thinspace.
	\]
	We can now bound $\abs{f^s_\pi(\gamma) - f^s_\pi(\tgamma)}$ as follows.
	\begin{align*}
		\abs{f^s_\pi(\gamma) - f^s_\pi(\tgamma)} &\le M \frac{\delta}{(1-\gamma)^2(1-\tgamma)} + M \frac{1}{1 - \gamma} \cdot \frac{\delta}{(1-\gamma)^2(1-\tgamma)}\\ &\le \frac{2 M \delta}{(1-\gamma)^3 (1-\tgamma)}\thinspace.
	\end{align*}
	Summing the above bound over the $|S|$ states we get the desired bound. 
\end{proof}

\existEpsSPE*

\begin{proof}
	First, we have already proven that in the case $\gamma^* \in [0,1] \setminus \Gamma$, there exists an exact SPE. Hence, in what follows, we assume that $\gamma^* \in \Gamma$.
	
	Now instead of finding a time step $T$ such that in the subsequent time steps $g(t)$ will surely lie between two neighbouring points in $\Gamma$, we find a $T$ such that $\abs{g(t)- g(T)} \le \delta$ for all $t \ge T$, where we let $\delta$ be a number such that 
	\[
	\delta \le \epsilon \cdot \frac{1}{4 M \cdot |S|} \cdot \left(\frac{1 - \gamma^*}{2} \right)^4
	\]
	and 
	\[
	\delta \le \frac{1 - \gamma^*}{2}\thinspace.
	\]
	By the convergence assumption and the fact that $1 \notin \Gamma$ by definition, there indeed exists such a number $\delta >0$.
	It then holds for all $t \ge T$ that
	\begin{align}
		\label{eq:exist-eps-SPE-1-max-gt-gT}
		1 - \max\{g(t), g(T)\} \ge 1 - (\gamma^* + \delta) \ge \frac{1 - \gamma^*}{2}\thinspace.
	\end{align}
	Hence,
	\begin{align}
		\abs{g(t) - g(T)} 
		&\le \epsilon \cdot \frac{1}{4 M \cdot |S|} \cdot \left(\frac{1 - \gamma^*}{2} \right)^4 \nonumber\\
		&\le \epsilon \cdot \frac{1}{4 M \cdot |S|} \cdot \left(1 - \max\{g(t), g(T)\} \right)^4 \nonumber\\
		&\le \cdot \frac{\epsilon}{4 M \cdot |S|} \cdot \left(1 - \max\{g(t), g(T)\} \right)^3 \cdot \left(1 - \min\{g(t), g(T)\} \right)\thinspace.
		\label{eq:exist-eps-SPE-gt-gT-le}
	\end{align}
	According to \Cref{lmm:eSPE-bound}, this means that
	\[
	\abs{V_{g(t)}^{\pi}(s) - V_{g(T)}^{\pi}(s)}
	\le \epsilon /2 
	\]
	for any policy $\pi \in \Pi$.
	In particular, let $\varpi^* \in \Pi_{g(t)}^*$ and $\pi^* \in \Pi_{g(T)}^*$ be the optimal policies with respect to constant discount factors $g(t)$ and $g(T)$, respectively.
	Applying the above inequality to these two policies gives
	\begin{align*}
		\abs{V_{g(t)}^{\varpi^*}(s) - V_{g(T)}^{\pi^*}(s)}
		\le& \abs{V_{g(t)}^{\varpi^*}(s) - V_{g(T)}^{\varpi^*}(s)} + \abs{V_{g(t)}^{\pi^*}(s) - V_{g(T)}^{\pi^*}(s)} \le \epsilon\thinspace.
	\end{align*}
	Hence, $\pi^*$ is a near-optimal policy for all players $t \ge T$.
	We can then construct an $\epsilon$-SPE $\ppi = (\pi_t)_{t=0}^\infty$ by first letting $\pi_t = \pi^*$ for all $t \ge T$, and then use Algorithm~\ref{alg:construct-SPE} to construct the preceding part of $\ppi$ through backward induction.
\end{proof}

\subsection{When Limit Point of $g$ Not Bounded Away From $1$}

We first prove the following lemma, which provides a lower bound for distances between points in $\Gamma \cup \{0,1\}$.
Intuitively, we can use this bound to locate a time step $t$ such that the tail of $g$ falls in between the largest point in $\Gamma$ and $1$. Hence, the optimal policy with respect to $g$ is eventually constant and we can construct an exact SPE for the subgame after $t$ and use backward induction to construct policies for time steps before $t$.

\begin{restatable}{lemma}{rootDistance}
	\label{lmm:root-distance}
	For any $\gamma_1, \gamma_2 \in \Gamma \cup \{0,1\}$ such that $\gamma_1 \neq \gamma_2$, we have 
	\begin{align}
		\label{eq:D}
		\abs{\gamma_1 - \gamma_2} \ge D \coloneqq 
		(n \cdot m)^{- b \cdot (n \cdot m)^{n^5}}\thinspace,
	\end{align}
	where $n = \abs{S}$ and $m = \abs{A}$ are large enough.
\end{restatable}

\begin{proof}
	Recall the expression for $f^s_\pi(\gamma)$ stated in \eqref{eq:f-pi-s-gamma-2}:
	\[
	f^s_\pi(\gamma) = \E[\gamma^{K^1_s}] \cdot R(s,\pi(s) ) \cdot \frac{1}{1 - \E[\gamma^{K^2_s}]}\thinspace.
	\]
	Moreover, according to \eqref{eq:E-gamma-K-s-2}, we have $\E[\gamma^{K^2_s}] = P(s,\pi(s), s) \cdot \gamma + p_s \gamma^2 (I - M\gamma)^{-1} q^s$. Using Cramer's rule we can write the matrix $(I-M\gamma)^{-1}$ as an $(|S|-1)\times(|S|-1)$ matrix with entries $C_{i,j}/\textrm{Det}(I - M\gamma)$ where $C_{i,j}$ is the $(i,j)$-th cofactor of the matrix $I-M\gamma$. This implies that each entry of the matrix $I-M\gamma$ is a rational polynomial. Moreover, as the determinant of an $n\times n$ matrix is the sum over all $n!$ permutations of matrix elements, and each entry of the matrix $I-M\gamma$ is bounded between $[-1,1]$, both the numerator and the denominator of such an entry have degree at most $\abs{S}$ and coefficient bounded between $[-\abs{S}!, \abs{S}!]$. 
	
	Additionally, each entry of $M$ can be represented by at most $b$ bits. Since multiplication (resp. addition) of two fractions of sizes $b_1$ and $b_2$ can be represented using $b_1 + b_2$ (resp. $\max\{b_1 ,b_2\} + 1$) bits, each coefficient of the above polynomials can be represented using at most $b\abs{S}\abs{S}!$ bits. Therefore, we conclude that each entry of $(I-M\gamma)^{-1}$ is a rational polynomial where both the numerator and the denominator have degree at most $|S|$, maximum coefficient at most $|S|!$ and maximum size of each coefficient is at most $b\abs{S}\abs{S}!$.
	
	Now the polynomial $p_s \gamma^2 (I-M\gamma)^{-1} q^s$ is constructed through vector-matrix multiplication and vector-vector inner product. So the final result is a rational polynomial where the denominator polynomial from $(I-M\gamma)^{-1}$ remains unchanged. Moreover, the numerator polynomial now has degree at most $|S|+2$. Since $\norm{p_s}_1\le 1$ and $\norm{q_s}_1 \le 1$, the maximum value of a coefficient  of the numerator polynomial does not increase. Finally, the polynomial $p_s (I-M\gamma)^{-1} q^s$ can also be written as $\sum_{i,j} p_s(i) (I-M\gamma)^{-1}(i,j) q^s(j)$. Each entry in the summation is a multiplication between a coefficient of size at most $b\cdot {\abs{S}}\abs{S}!$ and another coefficient of size at most $b$. Therefore, the maximum size of any coefficient in the denominator is at most $\abs{S}^2 + b(\abs{S}\cdot \abs{S}! + 1) \le b \abs{S}^{\abs{S}+2}$ as long as we have $\abs{S} \ge 3$. 
	
	By a similar argument as above, $\E[\gamma^{K^1_s}]$ can be represented as a rational polynomial where both the numerator and the denominator have degree at most $\abs{S}+2$, maximum coefficient at most $\abs{S}!$, and maximum size of each coefficient at most $b \abs{S}^{\abs{S}+1}$. It is also easy to see that  the polynomial $1/(1-\E[\gamma^{K^2_s}])$ is similar. Now $f^s_\pi(\gamma)$ is constructed by multiplying two such rational polynomials, and the degree of both the denominator and the numerator polynomial is at most $2(\abs{S}+2)$. 
	Each coefficient of the new polynomial is the result of adding at most $\abs{S}+2$ pairwise products. Therefore, the maximum value of any coefficient is at most $(\abs{S}+2)(\abs{S}!)^2 \le \abs{S}^{2\abs{S}+3}$ as long as $\abs{S}\ge 3$. And the maximum size of any coefficient is at most $(\abs{S}+2)b \abs{S}^{\abs{S}+2} + (\abs{S}+2) \le b\abs{S}^{\abs{S} + 4}$. 
	
	Now recall that we are interested in zeros of the following function (i.e., see \eqref{eq:h-s} and \eqref{eq:V-f}):
	\[
	h^s_{\pi_1, \pi_2}(\gamma) = \sum_{s' \in S} f^{s'}_{\pi_1}(\gamma) - f^{s'}_{\pi_2}(\gamma)\thinspace.
	\]
	Since the polynomial $h^s_{\pi_1,\pi_2}(\gamma)$ is constructed by adding/subtracting $2\abs{S}$ rational polynomials, we will apply the following argument $\log_2(2\abs{S})$ times.
	Suppose we add two rational polynomials of degree at most $d$, maximum value of any coefficient at most $v$, and maximum size of any coefficient at most $s$. Then the resulting rational polynomial will have degree at most $2\cdot d$, maximum value of any coefficient at most $dv^2$, and maximum size of any coefficient at most $ds+d \le 2ds$. 
	Repeating this process $j$ times, we get a polynomial with degree at most ${2^j}\cdot d$, maximum value of any coefficient at most $(3dv)^{2^j}$. In order to determine the maximum value of any coefficient, let $d_j$ (resp $v_j$) be the degree (resp. maximum value) after the $j$-th iteration. Then we have $v_j = d_{j-1}v_{j-1}^2 = d_{j-1} \left(d_{j-2} v_{j-2}^2 \right)^2 = \ldots = d_{j-1} d_{j-2}^2 d_{j-3}^4 \ldots d_0^{2^{j-1}} v_0^{2^j}$. Substituting $d_0=d, v_0=v$, and $d_j = 2^j \cdot d$, we get 
	\begin{align}
		v_j &= 2^{(j-1) + (j-2)\cdot 2 + (j-3) \cdot 4 + \ldots + 1 \cdot 2^{j-2}  } d^{1 + 2 + 4 + \ldots + 2^{j-1}} v^{2^j} \nonumber\\
		&\le 2^{j \cdot 2^j} (dv)^{2^j}\thinspace.
	\end{align}
	By a similar argument, we can show that the maximum size of any coefficient is at most $2^{j^2} d^j s$. Substituting $j = \log_2(2\abs{S})$, $d=2(\abs{S} + 2)$, $v=\abs{S}^{2\abs{S}+3}$, and $s=b\abs{S}^{\abs{S}+4}$, we get that $h^s_{\pi_1,\pi_2}(\gamma)$ is a rational polynomial where both the denominator and the numerator polynomial have the following guarantee -- each polynomial has degree at most $8 \abs{S}^{2}$, maximum value of any coefficient at most $2^{6\abs{S}\cdot \log_2 \abs{S}} \cdot \abs{S}^{4\abs{S}^2 + 8\abs{S}}$, and maximum size of any coefficient at most $b\cdot 2^{2(\log(2\abs{S}))^2}\cdot \abs{S}^{2\abs{S}+6}$.
	For simplicity, for large enough $|S|$, these numbers are bounded by $\abs{S}^3$, $\abs{S}^{\abs{S}^3}$, and $b \cdot \abs{S}^{\abs{S}^2}$, respectively.
	
	In order to bound the gap between zeros of $h^s_{\pi_1,\pi_2}(\gamma)$ for all $s \in S$ and all non-equivalent $\pi_1, \pi_2 \in \Pi$, we need to further consider the function 
	\[H(\gamma)\coloneqq \prod_{\pi_1,\pi_2 \in \Pi: \pi_1 \not\sim \pi_2} \prod_{s \in S} h^s_{\pi_1,\pi_2}(\gamma)\thinspace.\]
	Since there are at most $\abs{A}^{\abs{S}}$ many policies in $\Pi$, the bounds we derived further blow up to $\abs{S}^4 \cdot \abs{A}^{\abs{S}}$, $\abs{S}^{\abs{S}^4 \cdot \abs{A}^{\abs{S}}}$, and $b \cdot \abs{S}^{\abs{S}^3} \cdot \abs{A}^{\abs{S}}$, respectively.
	
	We now apply the following result from \citet{mignotte1982some}. Given a degree $d$ polynomial $P$ with integer coefficients and maximum coefficient $v$, its roots are separated by at least 
	\[
	\textrm{sep}(P) \ge \frac{\sqrt{3}}{d^{(d+2)/2} \norm{P}^{d-1}}\thinspace,
	\]
	where $\norm{P} \le \sqrt{d}v$. We wish to find the zeros of the numerator polynomial of $H(\gamma)$. However, the coefficients need not be integers. So we multiply each coefficient by $L$, the least common multiple of the denominators of all the coefficients of the numerator polynomial. Since each coefficient has size at most $b \cdot \abs{S}^{\abs{S}^3} \cdot \abs{A}^{\abs{S}}$, its denominator can be at most $2^{b \cdot \abs{S}^{\abs{S}^3} \cdot \abs{A}^{\abs{S}}}$. Moreover, the degree of $H(\gamma)$ is at most $ \abs{S}^{4} \cdot \abs{A}^{\abs{S}}$. This gives us the following bound on the least common multiple $L$.
	\begin{equation}\label{eq:bound-L}
		L \le \left( 2^{b \cdot \abs{S}^{\abs{S}^3} \cdot \abs{A}^{\abs{S}}} \right)^{\abs{S}^{4} \cdot \abs{A}^{\abs{S}}} \le 2^{b\abs{S}^{\abs{S}^4} \abs{A}^{\abs{S}^2}}\thinspace.
	\end{equation}
	Let us write $P_L$ to denote the polynomial $H(\gamma)$ multiplied by $L$. Then the degree of $P$ is at most $\abs{S}^4 \cdot \abs{A}^{\abs{S}}$ (say $d_L$) and maximum value of any coefficient is at most $L \cdot \abs{S}^{\abs{S}^4 \cdot \abs{A}^{\abs{S}}}$ (say $v_L$). Now we can apply Mignotte's result~\citep{mignotte1982some} to get the following bound:
	\begin{align*}
		\textrm{sep}(H) = \textrm{sep}(P_L) &\ge \sqrt{3}d_L^{-(d_L+2)/2 + (1-d_L)/2} v_L^{1-d_L} \\
		&\ge  (\abs{S} \cdot \abs{A})^{- b \cdot (\abs{S} \cdot \abs{A})^{|S|^5}}\thinspace.
	\end{align*}
	This completes the proof.
\end{proof}


\begin{algorithm}[t]
	\caption{Computing an $\epsilon$-SPE}\label{alg:computing-eSPE-unknown-gap}
	\SetKwInOut{Input}{Input}
	\SetKwInOut{Output}{Output}
	
	\Input{$\epsilon > 0$}
	\Output{an $\epsilon$-SPE $\ppi = (\pi_t)_{t=0}^\infty$}
	
	\BlankLine
	
	Set $D$ according to \eqref{eq:D};
	
	$T \leftarrow \mathcal{A}(D/4)$;
	
	\uIf{$\abs{g(T)-1} < D/2$}
	{
		$\tilde{\pi} \leftarrow$ an arbitrary policy in $\Pi_{g(T)}^*$;
	}
	\Else{
		$\delta \leftarrow \epsilon \cdot \frac{1}{4 M \cdot |S|} \cdot \left(\frac{D}{8} \right)^4$;
		
		$T' \leftarrow \mathcal{A}(\delta)$;
		
		$\tilde{\pi} \leftarrow$ an arbitrary policy in $\Pi_{g(T')}^*$;
	}
	
	$\pi_t \leftarrow \tilde{\pi}$ for all $t = T, T+1,\dots$;
	
	Run Algorithm~\ref{alg:construct-SPE} on input $\tilde{\pi}$ to construct $\pi_0,\dots,\pi_{T-1}$;
\end{algorithm}

Now we prove a more general result than \Cref{thm:eps-SPE-run-time}, where $\vargamma$ does not need to converge to a point sufficiently far away from 1. The method used for this is outlined in Algorithm~\ref{alg:computing-eSPE-unknown-gap}.
\begin{theorem}
	Suppose that $\vargamma$ is $(\alpha,\beta)$-convergent.
	An $\epsilon$-SPE can be computed in time $\left(\beta(|\vargamma|, d) \right)^2 \cdot \alpha(|\vargamma|, d) \cdot \poly(m,n)$, 
	where $d = b \cdot (n\cdot m)^{n^5} \log(n\cdot m)$, $m=|A|$, and $n = |S|$.
\end{theorem}

\begin{proof}
	As illustrated in Algorithm~\ref{alg:computing-eSPE-unknown-gap}, we first check whether or not $g$ converges to a point sufficiently far away from $1$.
	We run the oracle $\mathcal{A}$ on input $D/4$, and suppose the output is $T$;
	hence, 
	\[
	\abs{\vargamma(t) - \vargamma(T)} \le D/4
	\] 
	for all $t \ge T$.
	
	\paragraph{Case 1.}
	If $\abs{g(T) - 1} < D/2$, then we know that for all $t \ge T$,
	\[
	\abs{g(t) - 1} 
	\le \abs{g(t) - g(T)} + \abs{g(T) - 1}
	\le 3D/4
	< D\thinspace.
	\]
	Hence, according to \Cref{lmm:root-distance}, $g(t)$ must lie between the largest element in $\Gamma$ and $1$, and as a result $\Pi_{g(t)}^* = \Pi_{g(T)}^*$ for all $t \ge T$.
	We can pick an arbitrary $\tilde{\pi} \in \Pi_{g(T)}^*$, and let $\pi_{t} = \tilde{\pi}$ for all $t \ge T$ in the SPE $\ppi$ to be constructed.
	
	\paragraph{Case 2.}
	If $\abs{g(T) - 1} \ge D/2$, it follows that 
	\begin{align}
		\label{eq:comute-eps-SPE-gt-1-Dby4}
		\abs{g(t) - 1} 
		\ge \abs{g(T) - 1} - \abs{g(t) - g(T)}
		\ge D/4
	\end{align}
	for all $t \ge T$.
	Let
	\begin{align}
		\label{eq:comute-eps-SPE-delta}
		\delta = \epsilon \cdot \frac{1}{4 M \cdot |S|} \cdot \left(\frac{D}{8} \right)^4\thinspace,
	\end{align}
	where $M$ is the bound of the rewards in the MDP.
	We run $\mathcal{A}$ on input $\delta$ to determine a time step $T'$, which gives 
	\[
	\abs{\vargamma(t) - \vargamma(T')} \le \delta\thinspace.
	\] 
	According to \eqref{eq:comute-eps-SPE-gt-1-Dby4} and \eqref{eq:comute-eps-SPE-delta}, this also means that
	\begin{align*}
		\abs{\vargamma(t) - \vargamma(T')} 
		&\le \epsilon \cdot \frac{1}{4 M \cdot |S|} \cdot \left(\frac{D}{8} \right)^4\\
		&\le \epsilon \cdot \frac{1}{4 M \cdot |S|} \cdot \left(\frac{1 - \gamma^*}{2} \right)^4\thinspace,
	\end{align*}
	and moreover
	\[
	\abs{\vargamma(t) - \vargamma(T')}
	\le \frac{1 - \gamma^*}{2}
	\]
	if we assume without loss of generality that $\epsilon \le M$.
	It then follows by \eqref{eq:exist-eps-SPE-1-max-gt-gT}, \eqref{eq:exist-eps-SPE-gt-gT-le}, and \Cref{lmm:eSPE-bound} that any $\tilde{\pi} \in \Pi_{g(T')}^*$ is a near-optimal policy for player $t \ge T'$, with error at most $\epsilon$ to the optimum.
	
	Therefore, in both cases, we let $\pi_{t} = \tilde{\pi}$ for all $t \ge T$ in the SPE $\ppi$ to be constructed, and use Algorithm~\ref{alg:construct-SPE} to construct the remaining preceding part of $\ppi$.
	The whole process is outlined in Algorithm~\ref{alg:computing-eSPE-unknown-gap}.
	The run time of the algorithm follows immediately.
	Specifically, it takes time $\alpha(|\vargamma|, \poly(m,n, \log(1/\epsilon)))$ to run $\mathcal{A}$, where the input used is either $D/4$ or $\delta = \epsilon \cdot \frac{1}{4 M \cdot |S|} \cdot \left(\frac{D}{8} \right)^4$ (as specified in \eqref{eq:comute-eps-SPE-delta}), the size of the binary representation of which is $d = b \cdot (n\cdot m)^{n^5} \log(n\cdot m)$.
	In addition to that, the time it takes to run Algorithm~\ref{alg:construct-SPE} is bounded by $\left(\beta(|\vargamma|, d) \right)^2 \cdot \poly(m,n)$.
\end{proof}

\end{document}